\documentclass{llncs}  % Comment this line out if you need a4paper

\usepackage{algorithm}
\usepackage{algorithmic}
\usepackage{graphicx}                                                     % you want to use the \thanks
\usepackage{todonotes}
\newcommand{\REMOVED}[1]{{}}

\title{\LARGE \bf Optimal Distributed Searching in the Plane with and without Uncertainty}

\author{Alejandro L\'opez-Ortiz \and Daniela Maftuleac}

\institute{Cheriton School of Computer Science\\
 University of Waterloo\\
            Waterloo, ON N2L 3G1, Canada\\
        \email{\{alopez-o,dmaftule\}@uwaterloo.ca}}%

\begin{document}

\maketitle
%\thispagestyle{empty}
%\pagestyle{empty}

%%%%%%%%%%%%%%%%%%%%%%%%%%%%%%%%%%%%%%%%%%%%%%%%%%%%%%%%%%%%%%%%%%%%%%%%%%%%%%%%
\begin{abstract}
We consider the problem of multiple agents or robots searching for a target in the plane.
This is motivated by Search and Rescue operations (SAR) in the high seas which
in the past were often performed with several vessels, and more recently by
swarms of aerial drones and/or unmanned surface vessels.
Coordinating such a search in an effective manner is a non trivial task.
In this paper, we develop first an optimal strategy for searching with $k$ robots
starting from a common origin and moving at unit speed.
We then apply the results from this model to more realistic scenarios
such as differential search speeds, late arrival times to the search effort and low probability
of detection under poor visibility conditions. We show that, surprisingly,
the theoretical idealized model still governs the search with certain
suitable minor adaptations.
%Lastly we give several search scenarios showing the cost effectiveness of such searches, deriving from lower cost and higher precision in the search.
%We consider the problem of a swarm of robots searching for a target at high seas.
%  We also consider the case of searchers starting from different locations or different
%  speeds. This problem has natural applications to search and rescue operations
%  in the high seas where grid pattern searches are commonly used.
%
%  We also consider probabilistic implications of imprecise searches, and show that surprisingly
%  a coarse grid may be preferable over a fine grid in such cases.
%
\end{abstract}

%%%%%%%%%%%%%%%%%%%%%%%%%%%%%%%%%%%%%%%%%%%%%%%%%%%%%%%%%%%%%%%%%%%%%%%%%%%%%%%%
\section{Introduction}

Searching for an object on the plane with limited visibility is often
modelled by a search on a lattice. In this case it is assumed that
the search agent identifies the target upon contact.
An axis parallel lattice induces the Manhattan or $L_1$ metric on the plane.
One can measure  the distances traversed by the search agent
or robot using this metric.
Traditionally, search strategies are analysed using the competitive
ratio used in the analysis of on-line algorithms. For a single
robot the competitive ratio
is defined as the ratio between the distance traversed by
the robot in its search for the target and the length of the shortest path
between the starting position of the robot and the target.
In other words, the competitive ratio measures the detour of the search
strategy as compared to the optimal shortest route.

In 1989, Baeza-Yates et al. \cite{Alpern,Baeza,Baeza2} proposed an optimal strategy
for searching on a
lattice with a single searcher with a competitive
ratio of $2n+5+\Theta(1/n)$ to find a point at an unknown distance $n$ from the
origin.
The strategy follows a spiral pattern exploring $n$-balls in increasing
order, for all integer $n$.
%Not surprisingly this strategy is similar to those
%used in search and rescue operations \cite{CoastGuard}.
This model has been historically used for search and rescue operations in the high seas
where a grid pattern is established and search vessels are dispatched
in predetermined patterns to search for the target \cite{CoastGuard,USCoastGuardAdd}.

Historically, searches were conducted using a limited number (at most a handful)
of vessels and aircrafts. This placed heavy constraints in the type of solutions
that could be considered, and this is duly reflected in the modern search and rescue
literature \cite{IMO1,IMO2,IMO3,Koopman}.

However, the comparably low cost of surface or underwater unmanned vessels allows
for searches using hundreds, if not thousands of vessels.\footnote{For example, the cost of an unmanned
search vehicle is in the order of tens of thousands of dollars \cite{ClearPath} which can be amortized
over hundreds of searches, while the cost of conventional
search efforts range from the low hundred thousands of dollars up to sixty million dollars for high profile
searches such as Malaysia Airlines MH370 and Air France 447.
This suggests that somewhere in the order of a few hundred to a few tens of thousands of robots can
be realistically brought to bear in such a search.} %
Motivated by this consideration,
%the current prevalence and availability
%of drones and unmanned surface vessels as well as other low cost search agents,
we study strategies for searching optimally in the plane with a given, arbitrarily large number of robots.

Additionally, the search pattern reflects probabilities of detection and discovery according to
some known distribution that reflects the specifics of the search at hand. For example, the search
of the \emph{SS Central America} reflected the probabilities of location using known survivor
accounts and ocean currents. These probabilities were included in the design stage of the search
pattern, with the ship and its gold cargo being successfully recovered in 1989 after more than 130 years
of previous unsuccessful search efforts~\cite{Stone}.
%Coordinating such a search in an effective manner is referred to as
%a ``difficult task'' in the search and rescue literature \cite{CoastGuard,sarscene,USCoastGuardAdd}.

%In this paper we solve first the searching in the plane problem under ideal conditions of detection
%with $k$ unit speed robots starting a coordinated search simultaneously from a common origin.

%Lastly we give
%several search scenarios showing the cost effectiveness of such searches, deriving from lower cost of robot search
%hardware and higher precision in the search.

%\subsection{Searching on a Lattice}
In this paper, we address the problem of searching in the plane with multiple agents under probability of detection and discovery.
We begin with the theoretical model for two and four robots of L\'opez-Ortiz and
Sweet \cite{LopezOrtizSweet} that abstracts out issues of visibility and differing
speeds of searchers.
Searching for an object on the plane with limited visibility is commonly
modelled by a search on a lattice. Under this setting, visual contact on
the plane corresponds to identifying the target upon contact on the grid.

%However, in real life a search strategy occurs in the presence of
%multiple agents, which join the search at different times (and
%often at different speeds) and under varying visibility conditions.

% Coordination of such searches
%is referred to as a ``difficult task'' in the search and rescue literature
%\cite{CoastGuard,sarscene}.
%This problem is further complicated if the number of search agents is
%substantially larger than the handful of vessels and crafts used
%in traditional searches.

\subsection{Summary of Results and Structure of the Paper}

We construct a theoretical model and give an optimal strategy for searching with $k$ robots
with unit speed, starting simultaneously from a common origin.

We then progressively enrich this model with practical parameters, specifically different search speeds,
different arrival times to the search effort and poor visibility conditions. We show
that the principles from the theoretical solution also govern the
more realistic search scenario under these conditions subject to a few minor adaptations.
Lastly, we deal with cases with a varying probability of location as well as
probability of detection (POD).

%We believe the results give evidence towards the fact that
%a robot-swarm-driven search-and-rescue operation outperforms traditional
%searches in terms of costs and ability to locate the target. As such the difficulty
%of a search-and-rescue effort moves from a slow laborious search to a task of coordination
%of many searchers and management of probabilities as the search
%progresses.

%\subsection{Structure of the paper}

We first consider the case where all searchers start from a common point which we
term the origin, and second, when they start from arbitrary points on the lattice.

Initially we consider the case where all $k$ searchers move at the same speed and
give a strategy for finding a target
with $k=4r$ searchers with a competitive ratio of $2n/k+5/k$, as well as a
lower bound for $k$ searchers of $2n/k+5/k$ for general $k$, which matches the
upper bound up to an $o(1/k)$ term.

This is then generalized to any number of robots (not just multiples of four)
and using the same ideas, we show that the techniques developed also generalize
to searchers with various speeds. Lastly, we show that the proposed theoretical strategy also governs a search under
actual weather conditions, in which there is a non-negligible probability of
the target being missed in a search. We use tables from the extensive literature
on SAR (Search-and-Rescue) operations to conduct simulations and give scenarios
in which the proposed strategy can greatly aid in the quest for a missing
person or object in a SAR setting~\cite{IMO1,IMO2,IMO3}.

\begin{figure}
\vspace{-6cm}
\noindent%
\begin{minipage}[b]{.32\linewidth}
\centering\mbox{\includegraphics[width=1.8\textwidth]{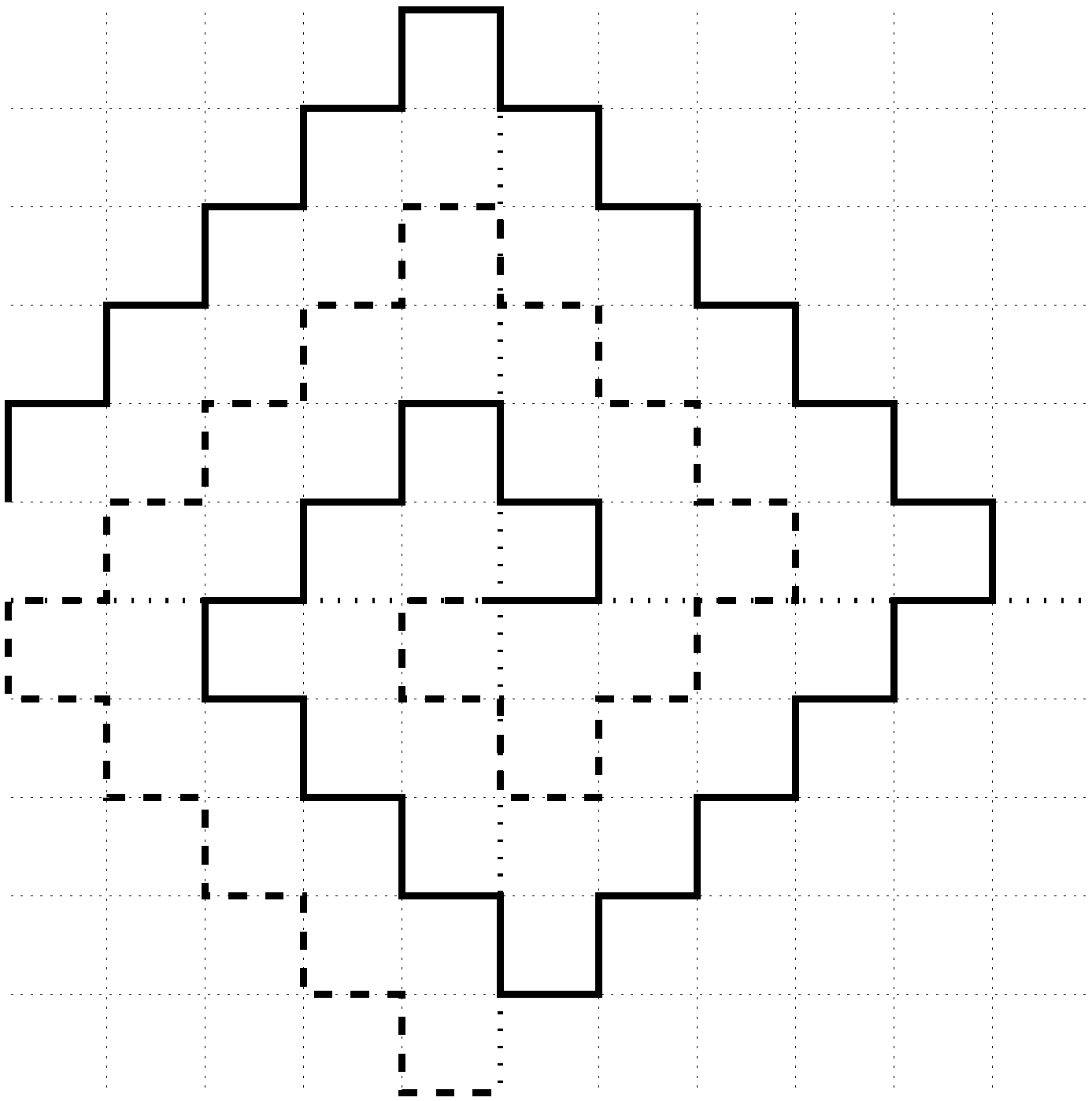}}
 \caption{ Search with two robots.} \label{2robots}
\end{minipage}\hspace{2cm}
\begin{minipage}[b]{.65\linewidth}
\centering\mbox{ \includegraphics[width=1.2\textwidth]{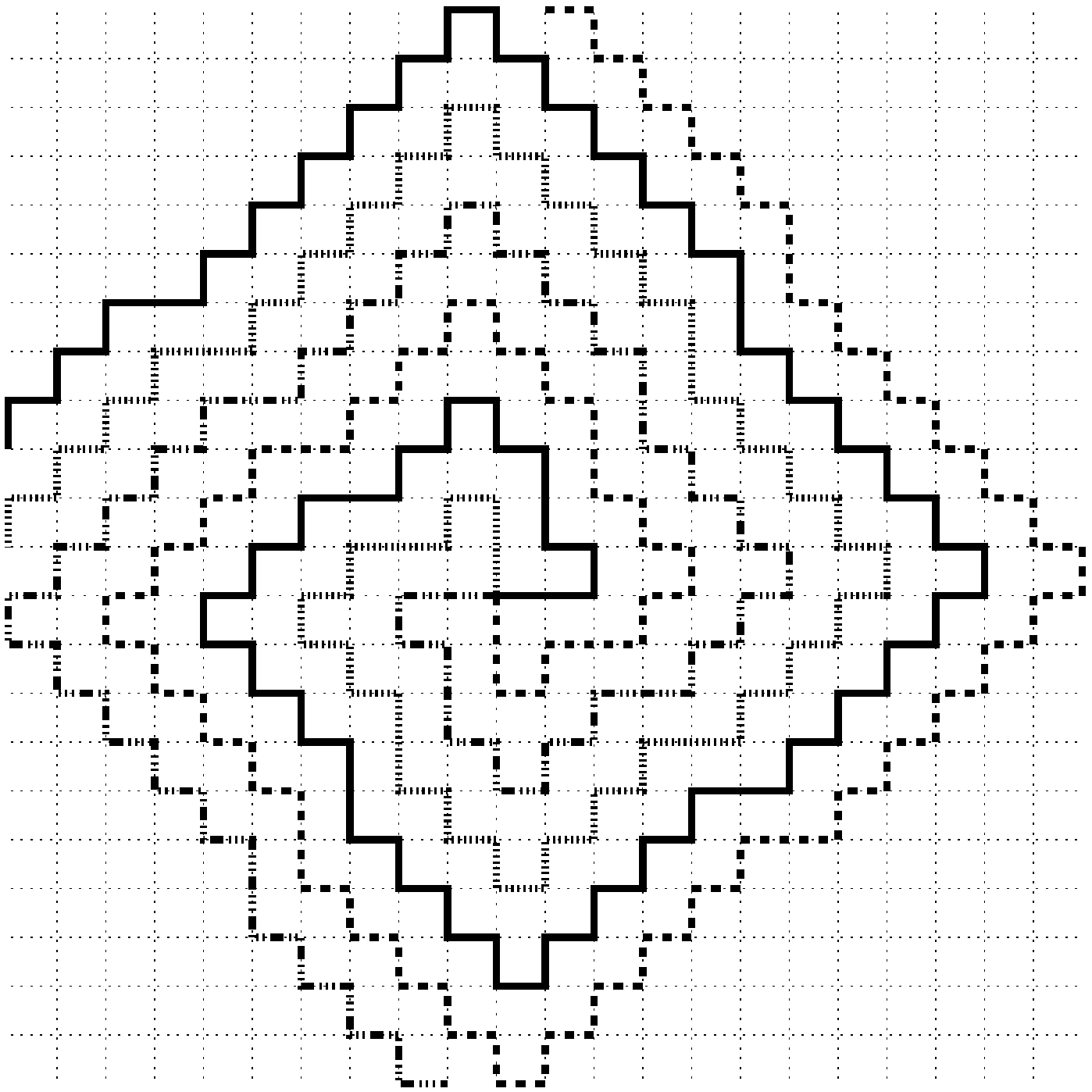}}
   \caption{Search with four robots.} \label{4robots}
\end{minipage}
\end{figure}

\section{Parallel Searching}

L\'opez-Ortiz and Sweet~\cite{LopezOrtizSweet} consider the case of searches using two and four robots
whose search path is shown in Figs.~\ref{2robots} and \ref{4robots}.
In this case, the robots move in symmetric paths around the origin and prove
the following theorem.
\begin{theorem} \cite{LopezOrtizSweet}
Searching in parallel with $k=2,4$ robots for a point at an unknown distance
$n$ in the lattice is $(2n+4+4/(3n))/k+o(1/n^2)$ competitive.
\end{theorem}
This is in fact optimal for the two and four robots case, as the next theorem shows.
\begin{theorem}\label{t:lower}
Searching in parallel with $k$ robots for a point at an unknown distance
$n$ in the lattice requires at least $(2n^2+4n+4/3)/k+\Omega(1/n)$
 steps, which implies a competitive ratio of at least $(2n+4+4/(3n))/k+\Omega(1/n^2)$.
\end{theorem}
\begin{proof}
Following the notation of \cite{LopezOrtizSweet},
let $A(n)$ be the combined total distance traversed by all robots
up and until the last point at distance $n$ is visited.
 We claim that in the worst case
 $A(n) \geq 2n^2+5n+3/2$, for some $n>1$.
Define $g(n)$ as the number of points visited on the $(n+1)$-ball
before the last visit to a point on the $n$-ball.
% Similarly
%$g(n)$ is defined as the number of points at a distance greater
%than $n+1$ before the last point at distance $n$ was explored.
%Define $f(n)$ as the number of points visited on the $(n+1)$-ball
%before the last visit to a point on the $n$-ball and $g(n)$
%as the number of points at a distance greater
%than $n+1$ before the last point at distance $n$ was explored.

First, note that there are $2n^2+2n+1$ points in the
interior of the closed ball of radius $n$ and that visiting
any $m$ points requires at least $m-1$ steps. Hence,
%$ A(n-1)\geq 2(n-1)^2+2(n-1)+f(n-1)+g(n-1).$
$A(n)= 2n^2+2n+g(n).$
If $g(n)$ points have already been visited, this means that after the last point at
distance $n$ is visited, there remain $4(n+1)-g(n)$ points to visit in the
$n$-ball. Now, visiting $m$ points in a ball requires at least $2m-1$ steps
with one robot, and $2m -k$ with $k$ robots. Thus visiting the remaining
points requires at least $2(4(n+1)-g(n))-k$ steps. Hence,
%$A(n)\geq A(n-1)^2+2\,(4n-f(n-1))-k
%\geq 2(n-1)^2+2(n-1)+ f(n-1)+g(n-1)+8n-2f(n-1)-k
%\geq 2(n-1)^2+2(n-1)-f(n-1)+8n-k\geq 2n^2+4n+(2-k)$
$A(n+1)= A(n)^2+2(4(n+1)-g(n))-k$
%&\geq& 2(n-1)^2+2(n-1)+ f(n-1)+g(n-1)+2(4n-f(n-1))-k\\
%&\geq& 2n^2+2n+8(n+1)-g(n)-k\\
%&\geq& 2(n-1)^2+2(n-1)-f(n-1) +8n-k\\
%&\geq& 2n^2+4n+(2-k)$
as claimed. Now we consider the competitive ratio at distance
$n$ and $n+1$ for each of the robots as they visit the last
point at such distance in their described path. We denote by
$A_i(n)$ the portion of the points $A(n)$ visited by the
$i$th robot. Hence, the competitive ratio for robot $i$ at
distances $n$ and $n+1$ is given by $A_i(n)/n$ and $A_i(n+1)/(n+1)$.
Observe that $\sum_{i=1}^k A_i(n)=A(n)$, for any $n$ and hence, there exist $i$ and
$j$ such that $A_i(n)\geq A(n)/k$ and $A_j(n+1)\geq A(n+1)/k$.
Lastly, the competitive ratio, as a worst case measure is minimized
when $A_i(n)/n=A_j(n+1)/(n+1)$ or equivalently when
$A(n)/n =A(n+1)/(n+1)$
with solution $g(n)=2n+(4-k)n/(3n+1).$
Substituting in the expression for $A(n)$, we obtain
$A(n)=2n^2+4n+(4-k)n/(3n+1)=2n^2+4n+4/3+\Theta(1/n)$
with a robot searching, in the worst case
at least $A(n)/k$ steps for a competitive ratio of
$$\frac{2n+4+4/(3n)}{k}+\Omega(1/n^2).$$
\qed
\end{proof}

%\begin{figure}
\noindent%
%\centering\mbox{\epsfig{file=Figures/8lat-low-bound.eps,height=9.4cm}}
%\caption{Eight robot search.}\label{f:8search}
%\end{figure}

\section{Search Strategy}

\subsection{Even-work strategy for parallel search with $k=4r$ robots}
A natural generalization of the $k=2$ and $k=4$ robot cases, as
shown in Figs. \ref{2robots} and \ref{4robots} suggest a spiral strategy
consisting of $k$ nested spirals searching in an outward fashion.
However, because the pattern must replicate or echo the shape of inner
paths, all attempts lead to an unbalanced distribution of the last search
levels and thus a suboptimal strategy.
A better competitive ratio gives us the strategy described in this section that
we call \emph{even-work strategy}.
Each of the $r$ robots covers an equal region of a quadrant using the pattern
in Fig.~\ref{step2}. The entire strategy consists of four rotations
of this pattern, one for each quadrant in the plane.
%A detailed description is available in the extended technical report version of
%this paper.\marginpar{register the paper with helen as a technical report.}
%\begin{figure}[h]\centering
% \includegraphics[width=0.5\textwidth]{Figures/step5.png}
%\caption{First quadrant of a parallel search with $k=4r$ robots, where $r=7$.} \label{one-quad-28-rob}
% \end{figure}
%The asymptotic c.r. for the region strategy for $r$ robots.
\begin{theorem}\label{t:upper}
Searching in parallel with $k=4r$ robots for a point at an unknown distance $n$ in the
lattice has the asymptotic competitive ratio of at most
%$\frac{2n}{k}+\frac{5}{k}\leq$
$(2n+7.42)/k.$
\end{theorem}
\begin{proof}
We know the lower bound for asymptotic competitive ratio is $2n/k + 5/k$. We want
to describe the upper bound of even-work strategy of $2n/k + 7.428/k$.
\begin{figure}[h]\centering
\includegraphics[width=0.5\textwidth]{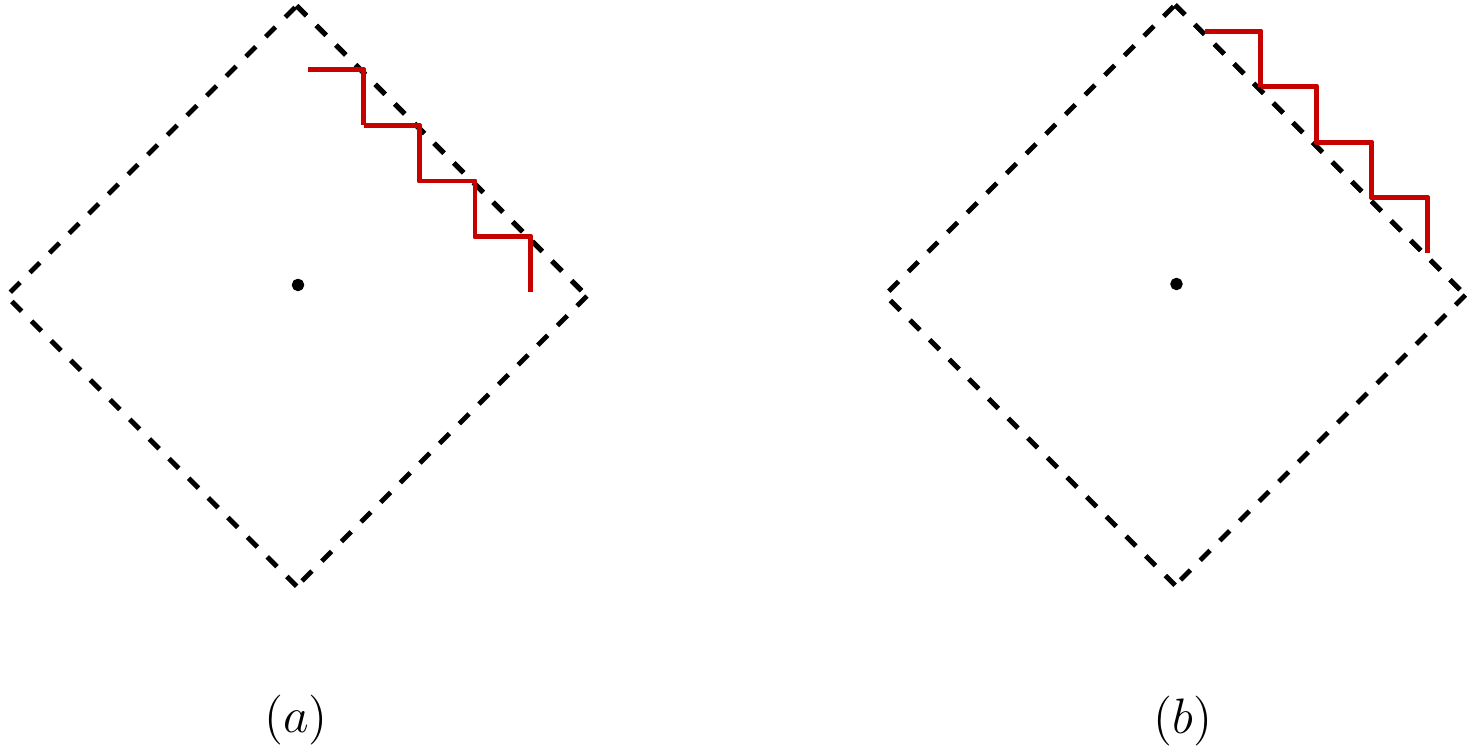}
\caption{Covering the $n$-ball: (a) best case scenario (b) worst case scenario.} \label{UB}
\end{figure}
From the lower bound, we can deduce that for each ball the number
of extra points covered by the robots is 5 in the best case (Fig. \ref{UB} (a)).
In the worst case, the robots perform 8 units of extra amount of work (Fig. \ref{UB} (b)).
So in order to cover all the points on a ball, the robots traverse a total of 13 units of extra distance.
Thus, 13/5 = 2.6  is an upper bound on the amount of work per point. When robots move from the
ball of radius $n$ to $n+1$, a single robot must pick up the extra point to be
explored. We balance the distribution of the new work as shown in Fig.~\ref{int-7r}.
\begin{figure}\centering
\includegraphics[width=0.9\textwidth]{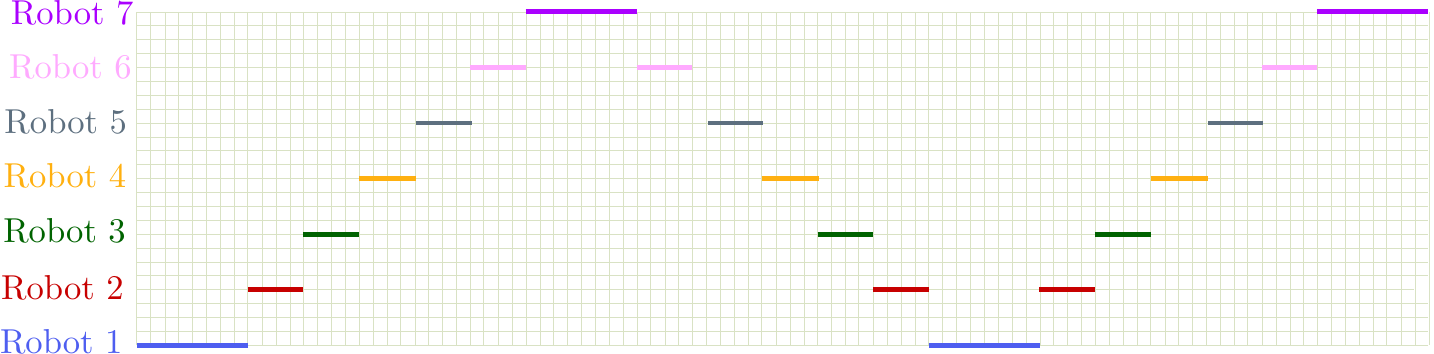}
\caption{Allocation of additional search tasks as radius increases (x-axis). The
y-axis indicates which robot is activated in that ball.} \label{int-7r}
\vspace{-0.6cm}
\end{figure}
After covering the ball $n$, we have $2n^2+2n$ points covered, this is if we
had the same amount of work for each point.
The lower bound gives us $2n^2+5n$ amount of work to cover all the points
at distance $n$.
When we look at the last $4n$ points (on the $n$ ball), for each of the $4n$ points, we have $3n$ work.
Thus, 7/4 amount of work per point (lower bound).
From where we get the relation:
$\frac{\lceil n/k\rceil (1+8/5)}{\lceil n/k\rceil (1+3/4)}= \frac{13/5}{7/4}=1.486.$
and $5\cdot 1.486 = 7.428.$\qed
\end{proof}

%%%%%%%%%%%%%%%%%%%%%%%%%%%%%%%%%%%%%%%%%%%%%%%%%%%%%%%%%%%%%%%%%%%%%%%%%%%%%%%%

\subsection{Parallel search with any number of robots}

This case illustrates how the abstract search strategy
for a number of robots multiple of four can readily be adapted to
an arbitrary number of robots.
Let $k$ be the number of robots, where $k$ is not necessarily
divisible by 4.

\begin{figure}[h]\centering
 \includegraphics[width=0.6\textwidth]{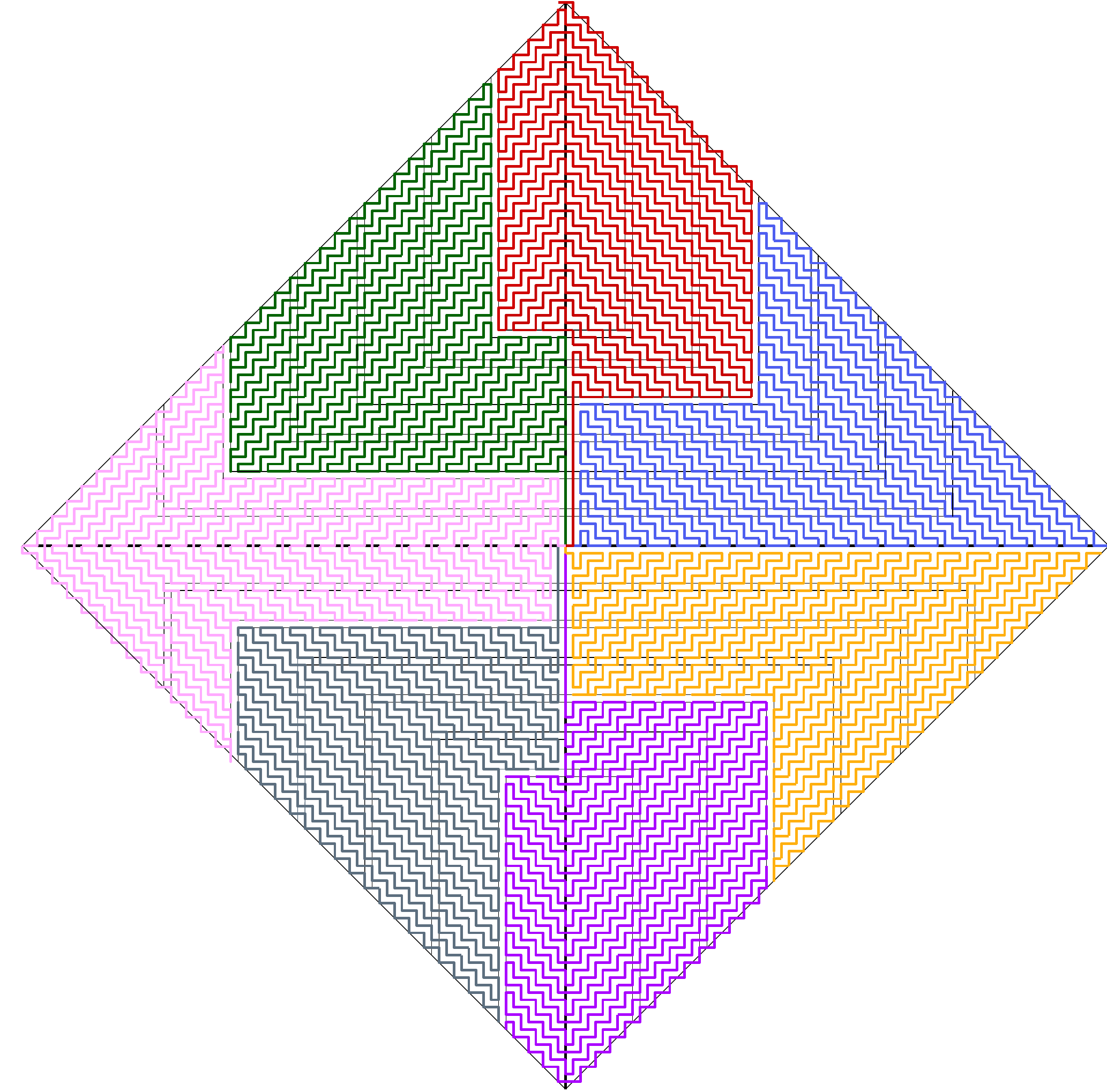}
\caption{Parallel search with 7 robots.} \label{7-rob}
 \end{figure}

%We adapted the even-work strategy for any number $r$ of robots.
We first design the strategy for $4k$ robots obtaining 4 times
as many regions as robots. We then assign to every robot 4
consecutive regions as shown in Fig.~\ref{7-rob} for the case
$k=7$. Observe that now some of the regions span more than one
quadrant and how the search path for each robot transitions
from region to region while exploring the same ball of radius
$n$ in all four regions assigned to it. Observe that from Theorems \ref{t:lower} and
\ref{t:upper},
it follows that this strategy searches the plane optimally as well.

\begin{theorem}\label{thm:any-k-robots}
Searching in parallel with $k$ robots for a point at an unknown distance in the lattice has an asymptotic competitive ratio of at most $(2n+7.42)/k$.
\end{theorem}

\section{From theory to practice}

\subsection{The Search Strategy}

%In Figs. \ref{step2} and
In Fig.~\ref{step2}, we show the search strategy
with $k=4r$ robots. %as it progresses in time.
The snapshots are taken at search time
$t=40,80,160$ and
$t=260$. Since the robots traverse at unit speed,
the total area explored by each robot is
$t$ while the combination of all robots is $kt$.
\begin{figure}[h]\centering
\begin{tabular}{cccc}
 \includegraphics[width=0.113\textwidth]{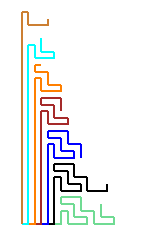}
& \includegraphics[width=0.149\textwidth]{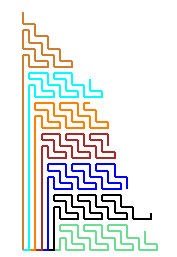}
& \includegraphics[width=0.22\textwidth]{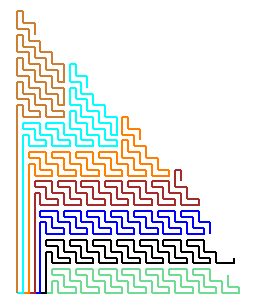}
& \includegraphics[width=0.37\textwidth]{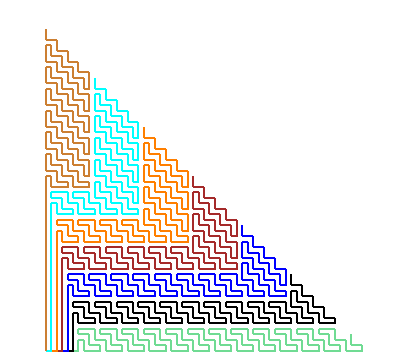}
\end{tabular}
\caption{Parallel search with $r=7$ robots at time $t=40,80, 160, 240$.} \label{step2}
 \end{figure}

%\begin{figure}[h]\centering
% \includegraphics[width=0.36\textwidth]{Figures/step4.png}
%\caption{Parallel search with $r=7$ robots at time $t=260$.} \label{step4}
% \end{figure}

%\subsection{Adding robots to the search}

\noindent While we envision the swarm of robots being usually deployed
from a single vessel and as such all of them starting from the
same original position, for certain searches additional resources
are brought to bear as more searchers join the search-and-rescue
effort. In this setting we must consider an agent or agents joining
a search effort already under way.

\begin{theorem}
There exists an optimal strategy for parallel search with $k$ initial
robots starting from a common origin and later adding new robots to the search.
\end{theorem}

\begin{proof}
 In this case we can compute
the exact time at which the additional searcher will meet up with
the explored area and have the search agents switch from a $k$
robot search pattern to a $k+1$ search pattern. The net cost of
this transition effort is bounded by the diameter of the
$n$ ball at which the extra searcher joins, with no ill effect
over the asymptotic competitive ratio. Hence, the search is
asymptotically optimal.
\qed
\end{proof}
%write about $k$ robots exploring from the origin and 4p additional
 %   robots joining the search.

%\subsection{Parallel search with $k$ robots with different speeds}

\noindent Parallel search with $k$ robots with different speeds
is another case which nicely illustrates how the abstract search strategy
for robots with equal speed can be readily adapted to robots of
varying speeds.
\begin{theorem}
There exists an optimal strategy for parallel search with
$k$ robots with different speeds.
\end{theorem}

\begin{proof}
Suppose we are given $k$ robots with varying speeds.
Let the speed of the $k$ robots be $s_1, s_2,\ldots, s_k$ respectively.
We can consider the speeds to be integral, subject to proper scaling and rounding.
%Write about $k$ robots exploring from the origin with different speeds
Let $s = \sum_{i=1}^k s_i$. We use the strategy for $4s$ robots
and we assign for each robot respectively: $4s_1, 4s_2, ..., 4s_k$ regions.
It follows that every robot completes the exploration of its region at the
same time as any other robot since the difference in area explored
corresponds exactly to the difference in search speed and the search
proceeds uniformly and optimally over the entire range as well.
\qed
\end{proof}

\subsection{Probability of detection}

In real life settings there is a substantial probability that the
search agent might miss the target even after exploring the immediate
vicinity of the target. Indeed, in searches on high or stormy seas often
multiple passes must be made before a man overboard is located
and rescued. In this case, the search vessel
uses a nautical pattern resembling a clover and is known as sector search.
(See~\cite{Lop:thesis,sarscene} for a robotics perspective of sector searches
and other SAR techniques).

The search-and-rescue literature provides ready tables of
probability of detection (POD) under various search conditions \cite{IMO1,IMO2,IMO3}.
Fig.~\ref{heatmap} shows the initial probability map for a typical man overboard event.
Fig.~\ref{probabilities} shows the probability of detection as a function
of the width of the search area spanned. The unit search width magnitude is
computed using location, time, target and search-agent specific information such as
visibility, lighting conditions, size of target and height of search vessel.
We consider then a setting in which a suitable POD distribution
has been computed taking into account present visibility conditions and
size of target (see Fig. \ref{heatmap}).
Armed with this information, a robot
must then make a choice between searching an unexplored
cell in the lattice or revisiting a previously explored cell.

\begin{figure}[h]\centering
 \includegraphics[width=0.45\textwidth]{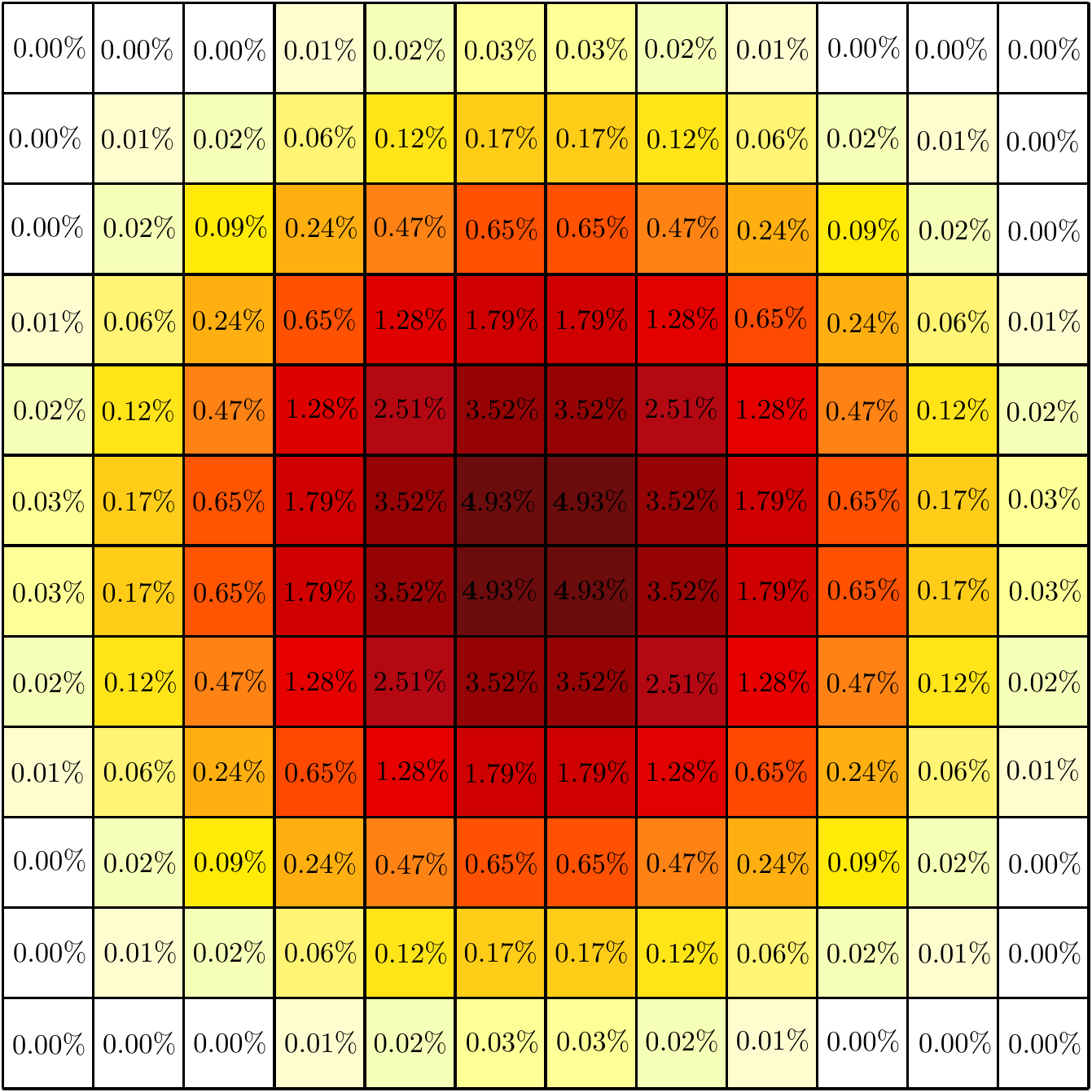}
\caption{Initial probability map \cite{IMO2}.} \label{heatmap}
 \end{figure}

Consider first a model in which a robot can ``teleport'' from any
given cell to another, ignoring any costs of movement related
to this switch.
The greedy strategy consists of robots moving to the cell with current highest
probability of containing the target. Each cell is then
searched using the corresponding pattern for the number of robots deployed in the cell.

\begin{lemma}\label{lem:teleport}
Greedy is the optimal strategy for searching a probabilistic space under the teleport model.
\end{lemma}

\begin{proof}
Let $p^j_i$ be the probability of discovering the target in cell $i$ during the $j$ visit, sorted in decreasing order.
We now relabel them $p_1, p_2, \ldots$ The expected time of discovery is $\sum_{t=1}^{\infty} p_t\cdot t$ which is minimized
when $p_t$ are in decreasing order, which can be shown formally via a standard
greedy technique proof.
%It can be shown formally using a standard
%greedy technique proof that in this setting the optimal strategy consists of
%robots moving greedily to the cell with current highest
%probability of containing the target.
\qed
\end{proof}

While we introduce the teleportation model as a means of understanding the
complexity of the problem, observe that in real life robots could potentially
be redeployed using means such as air lifting with orders of magnitude faster
traveling speeds than surface searching vessels. As such, the teleportation
cost would be a lower order term though never strictly zero.

In what follows we study more carefully the case where moving from one search position
to another happens at the same speed as searching and hence, transit time should be
accounted for in the algorithm.
\begin{figure}[h]\centering
 \includegraphics[width=0.58\textwidth]{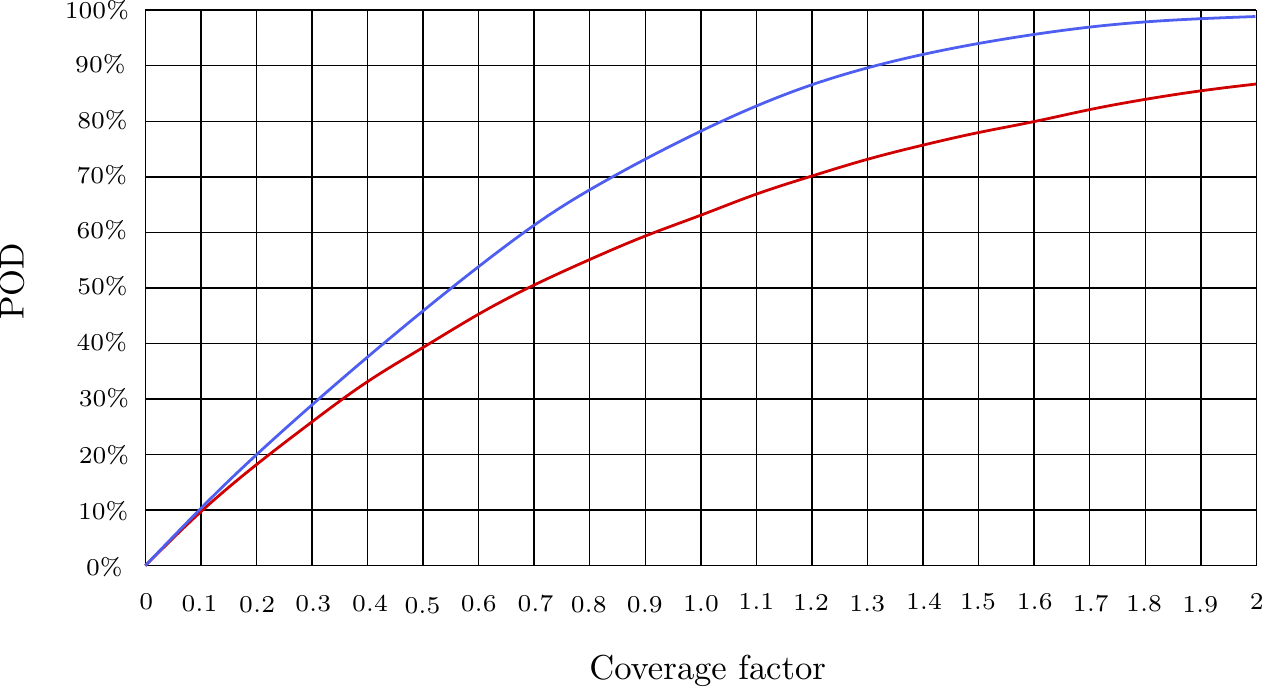}
\caption{Average probabilities of detection (POD) over an area for visual searches using parallel sweeps (blue line: ideal search conditions,
red line: normal search conditions)~\cite{IMO2}.} \label{probabilities}
 \end{figure}
Observe that the probability of each cell evolves over time.
It remains at its initial value
so long as it is still unexplored and it becomes
$(1-p)^m$ times its initial value after $m$ search passes, where $1-p$ is the probability of
not detecting a target present in the current cell during a single search pass.

\subsubsection{High Level Description of the Search Strategy.}
In real life, of course, there is a cost associated with moving from a cell to another.
In this case, the algorithm creates \emph{supercells} of size $h\times h$ unit cells, for some value of $h$, which
depends on the total number of robots available to the search effort.

The algorithm computes the combined probability of the target being found in a supercell which corresponds
to the sum of the individual probabilities of the unit cells as given by the POD map.

At each time $t$, the algorithm considers the highest probability key supercell %$C^t_{next}$
and compares it to the lowest probability key supercell being explored to determine if a robot transfer should
take place. If so, it updates the probabilities of discovery accordingly. The search process continues
ad infinitum or until the probability of finding the target falls below a certain
threshold, in which case we consider that the object is irretrievably lost.

The process above is used to compute the number of robots gained/lost by each
supercell. Once the probabilities have been rebalanced, we need to determine the source/destination
pair for each robot. This is important since the distance between source and
destination is dead search time, so we wish to minimize the amount of transit
time. To this end, we establish a minimum-cost network flow~\cite{comb-opt}
that computes
the lowest total transit cost robot reassignment that satisfies
the computed gains and losses.

\subsubsection{Probabilistic Search Algorithm.}

More formally, let $C^t_1, C^t_2,\ldots, C^t_j$ be the areas being explored at time $t$
by $r^t_1, r^t_2,\ldots, r^t_j$ robots respectively for $r_i \geq 0$. The combined
probability of a supercell is the sum of the probabilities of the cells inside it.

These combined probabilities are then sorted in decreasing order and the algorithm
dispatches robots to the highest probability supercell until the marginal value of the
robots is below that of an unexplored supercell. More precisely, let $C_i$ and $C_j$ be
the two supercells of highest combined probability, $p_i$ and $p_j$, respectively. The algorithm then assigns
$s$ robots to supercell $C_i$ such that
\begin{equation}
p_1/s\geq p_2 > p_1/(s+1). \label{eq:stasys}
\end{equation}
The algorithm
similarly assigns robots to $C_j$, $C_m$, and so on, updating the
discovery probability per robot for each supercell. Observe that if more than
one robot is assigned to $C_j$, the algorithm might need to increase the
assignment of robots to $C_i$ so that Ineq. \ref{eq:stasys} is
re-established.

This process continues until all $k$ robots have been allocated, and the search begins.
More specifically, the algorithm maintains two priority queues.
One is a max priority queue (PQ) of supercells using the combined probability per robot as key.
That is, supercell $i$ appears with priority key equals to $p_i/(r_i+1)$ where $r_i$ is the present
number of robots assigned to it by the algorithm.
The other is a min PQ of supercells presently being explored with the residual probability $p_i/r_i$ of each as key.

The algorithm then compares the top element in the maxPQ with the top element in the minPQ. If the
probability of the maxPQ is larger than the minPQ it adds an additional robot to the
maxPQ node, and decrements its key with updated priority. Similarly, the minPQ node losses a
robot and its priority is incremented due to the loss of one robot.

%keeps tabs of the residual probability of each explored supercell and
%updates a Min priority queue

As the cells within a supercell are explored, the probability that the target is
contained in the supercell is updated according to the probability of missing the
target in an explored cell.

When the probability of the least likely supercell currently
being explored falls below that of a supercell in the maxPQ we transfer a robot from
the present supercell to the maxPQ supercell. The minPQ supercell gets an increased probability
while the maxPQ supercell gets a decreased probability resulting from the respective
decrease/increase in the number of robots in the denominator. The keys are updated
accordingly in both priority queues and the algorithm continues transferring robots
from minPQ supercells to maxPQ supercells. The algorithm however, does not remove the last
robot from a supercell until all cells within it have been explored at least once.

\begin{figure}[h]\centering
 \includegraphics[width=0.6\textwidth]{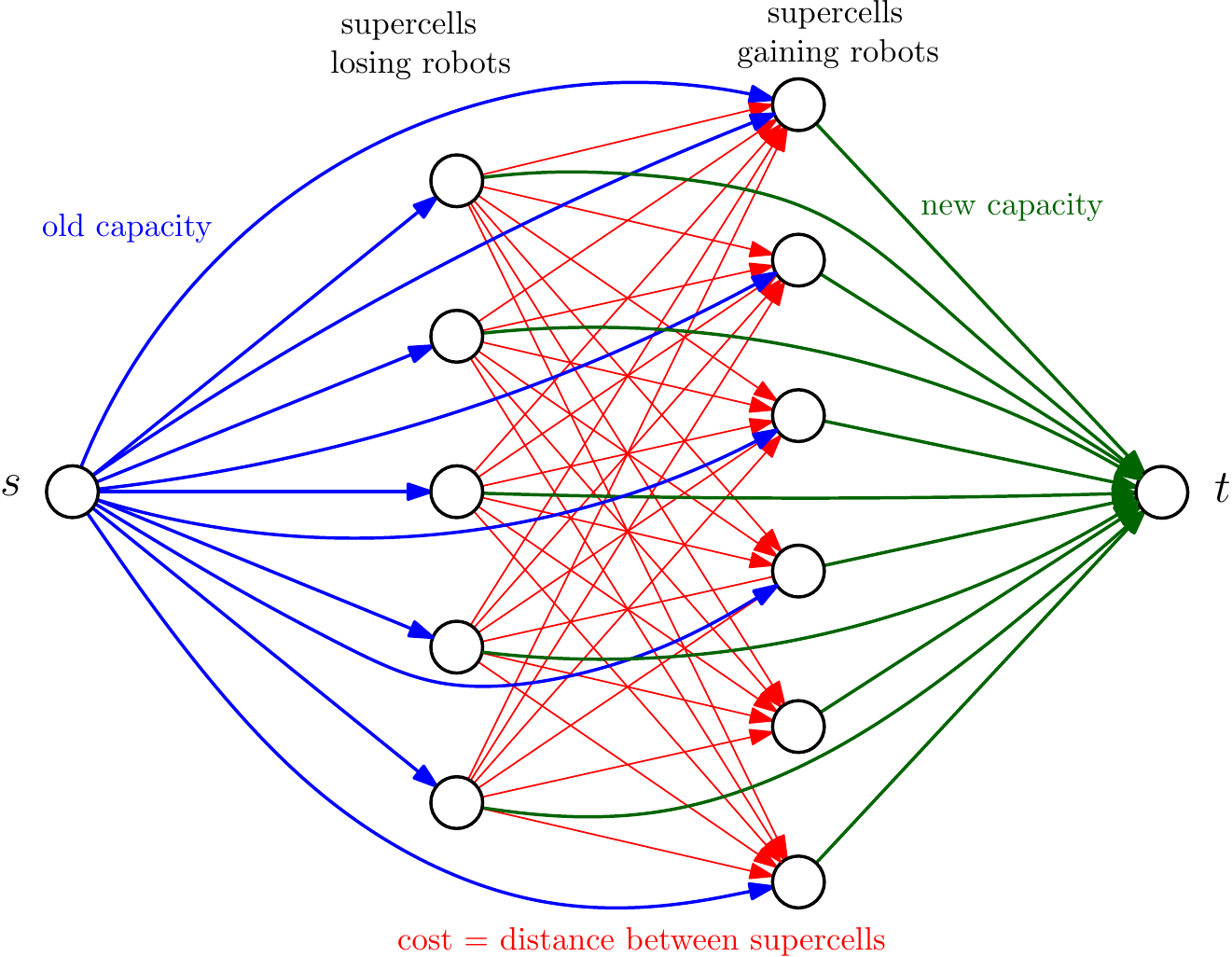}
\caption{Optimal robot reassignment via minimum cost network flow.} \label{min-flow}
 \end{figure}

Once the algorithm has computed the number of robots gained/lost by each
supercell, it establishes a minimum cost network flow problem to compute
the lowest total transit cost robot-reassignment schedule that satisfies
the computed gains and losses.

This is modelled as a network flow in a complete bipartite graph (see Fig.~\ref{min-flow}). In this
graph, nodes
% correspond to supercells. Nodes
on the left side of the bipartite graph correspond to supercells losing robots,
while nodes on the right correspond to supercells gaining robots.
Every node (both losing and gaining) has an incoming arc from the
source node with capacity equal to the old number of robots in the associated
supercell and cost zero. Similarly, all nodes are
connected to the sink with an edge of capacity exactly equal to the
updated robot count of the associated supercell and cost zero as well. Lastly, the cross edges
in the bipartite graph have infinite capacity and cost equal to the distance
between the supercells represented by the end points.

From the construction it follows that the only way to satisfy the constraints is to
reassign the robots from the losing supercell nodes to the gaining supercell nodes
at minimum travel cost.
This network flow problem can be solved in $O(E^2)$ time using the algorithm
of Orlin \cite{orlin}. In this case $E=O(n^2)$ and hence, in the
worst case the minimum cost network flow algorithm runs in time $O(n^4)$ where $n$ is the number of
supercells.

We summarize this algorithm as the following theorem:

\begin{theorem}
The probabilistically weighted distributed scheduling strategy for the time interval $[0,t]$ can be computed
in $O(t\, n^4)$ steps, where $n$ is the size of the search grid.
\end{theorem}

%. If $Prob(C^t_{next}) > \min_{j=1,\ldots,k}\{Prob(C^t_j)/r^t_j\}$
%then the algorithm redeploys one robot from that cell to $C^t_{next}$ and recomputes the
%gain for $C^t_1, C^t_2,\ldots, C^t_j$ with the updated robot count and repeats the
%re-evaluation above, further repositioning other robots.

%for at least one $j$ it considers redeploying a robot from that
%cell to It then sequentially considers sending
%$i=1, 2, \ldots$ robots to that cell with prob
\REMOVED{
Observe that there is a choice
as to how many robots to redeploy from the present explored position
to the next highest probability supercell or supercells. The optimal value can be obtained via
an iterative argument where we compute the probability of the cells explored
if we send one robot to the high probability cell and the rest remain in the
present expanding pattern. If the probability of the central area combined
exceeds that of the combined expanding search area it follows that the
strategy only improves if we send two robots towards the remote
high probability area. This process continues, considering 3 robots, then
4 and so on. Optimality is achieved at the point of equilibrium when
the probability of the lowest probability cells explored by the
robots changing cell are equal to the lowest probability cells explored
by the robots continuing on the expanding search pattern of Fig. \ref{7-rob}.
}

\subsection{Time to discovery}

There are several parameters of a SAR cost. First is the total cost of
the search effort as measured in vessel and personnel hours times the
number of hours in the search. The second is the effectiveness of the
search in terms of the probability of finding the target. Lastly,
the time to discovery or speed-to-destination as time is of the
essence in most search rescue scenarios. That is to say, a multiple
robot search is preferable to a single
agent search with the same cost and probability of detection as
time to discovery is lower. Observe as well that multiple robots
allow for higher coverage of the unit search width, dramatically increasing
the probability of detection. The strategies presented in this paper
suggest that robot swarm searches outperform traditional searches in
all three of these parameters.

\section{Conclusion}

We present optimal strategies for robot swarm searches under both idealized and
realistic considerations.
%We argue that such searches are also more cost effective, both from a net dollar
%perspective and in terms of speed of recovery for time critical searches.
The search strategies are based on a theoretical search primitive which is then
enriched with realistic considerations. Interestingly, the theoretical model
is resilient to these assumptions and can be readily adapted to take them
into consideration. We give pseudo-code showing that the search primitives
are simple and can easily be implemented with minimal computational and
navigational capabilities. We then give a heuristic to account for the probability
of detection map often available in real life searches.
The strategies proposed have a factor of $k$ improved time to discovery as compared
to a single searcher for the same total travel effort.


%\bibitem{c18} J. Williams, ÒNarrow-band analyzer (Thesis or Dissertation style),Ó Ph.D. dissertation, Dept. Elect. Eng., Harvard Univ., Cambridge, MA, 1993.
%\bibitem{c19} N. Kawasaki, ÒParametric study of thermal and chemical nonequilibrium nozzle flow,Ó M.S. thesis, Dept. Electron. Eng., Osaka Univ., Osaka, Japan, 1993.
%\bibitem{c20} J. P. Wilkinson, ÒNonlinear resonant circuit devices (Patent style),Ó U.S. Patent 3 624 12, July 16, 1990.

%%%%%%%%%%%%%%%%%%%%%%%%%%%%%%%%%%%%%%%%%%%%%%%%%%%%%%%%%%%%%%%%%%%%%%%%%%%%%%%%
\newpage
\section{Appendix}
Let $k = 4r$ be the number of robots searching in parallel starting from a common origin.
The following pseudo-code describes the algorithm of parallel search using $r$ robots for the first quadrant.
A simple rotation applied to the code gives the search strategy for the other three quadrants, which we
omit for reasons of clarity.
%=========================================================================
\begin{algorithm}
\caption{Strategy$(r,n)$}
\begin{algorithmic}
  \STATE {\bf Input:} Let $k=4r$ the number of robots, and let $n$ be the covered distance.
  \STATE {\bf Output:} parallel search strategy of $r$ robots in a quadrant.
        \STATE Robot-1($n$).
        \FOR{$i=2$ \TO $r-1$ }
            \STATE Middle-robots($i,\;n$).
        \ENDFOR
        \STATE Robot-r($n$).
\end{algorithmic}
\label{alg:Strategy}
\end{algorithm}

\begin{algorithm}
\caption{Robot-1($n$)}
\begin{algorithmic}
  \STATE {\bf Input:} Let $k=4r$ the total number of robots and let $n$ be the covered distance.
  \STATE {\bf Output:} The parallel search strategy of the first robot in a quadrant.
    \STATE Initialization($r$, 0);% coordinates from where the robot walks 2 up, 1 right, 2 down, 3 right.
     \FOR{$v=1$ \TO $n$} %how far the robot goes
        %stable horizontal band
         \FOR{$j=1$ \TO $2(r-1)$}
            \STATE Stairs($8(v-1)+3$, $horizontal$, $NW$). %3 11? ...
            \STATE Stairs($8(v-1)+5$, $horizontal$, $SE$). %5 13? ...
        \ENDFOR
         \FOR{$j=1$ \TO $2$}
           \STATE Stairs($4j + 8(v-1)$, $horizontal$, $NW$).% 4 8 | 12 16 | 20 24
           \STATE Stairs($4j+2 + 8(v-1)$, $vertical$, $SE$).%6 10 | 14 18 | 22 26
        \ENDFOR
    \ENDFOR
\end{algorithmic}
\label{alg:Robot-1}
\end{algorithm}

\begin{algorithm}
\caption{Initialization($x,y$)}
\begin{algorithmic}
  \STATE {\bf Input:} Let $(x,y)$ be the initial starting point.
  \STATE {\bf Output:} Constructs the initial pattern for a robot
    \STATE 2 up.
    \STATE 1 right.
    \STATE 2 down.
    \STATE 3 right.
\end{algorithmic}
\label{alg:Init}
\end{algorithm}

\begin{algorithm}
\caption{Stairs($n$, $d$, $direction$)} %-here we construct the walking of the robot on a part of a ball. We call Init-Stair to construct the actual staircases
\begin{algorithmic}
  \STATE {\bf Input:} Let $n$ be the number of steps in the stair, $d$ - the initial horizontal or vertical step and $direction$ either $NW$ for North-West
  or $SE$ for South-East.
  \STATE {\bf Output:} The stairs in direction $direction$ starting with the first step $d$.
   \IF{$direction = NW$}
        \STATE 1 up.
        \STATE Init-Stair($n$, $d$, $NW$).
        \STATE 1 up.
  \ELSE
        \STATE 1 right.
        \STATE Init-Stair($n$, $d$, $SE$).
        \STATE 1 right.
  \ENDIF
\end{algorithmic}
\label{alg:Stairs}
\end{algorithm}

\begin{algorithm}
\caption{Init-Stair($n$, $d$, $direction$)}% - here we construct the actual stairs. Example: 2 right, 2 up, 2 right, 2 down,... The length of the stairs is the number of pairs 2 right or 2 down. So after we construct a staircase (2 right or 2 down), we call recursively Init-Stair with n-1 length of the stairs
\begin{algorithmic}
  \STATE {\bf Input:} Let $n$ be the number of steps in the stair, $d$ - the initial horizontal or vertical step and $direction$ either $NW$ for North-West
  or $SE$ for South-East.
  \STATE {\bf Output:} The $n$ stairs in direction $direction$ starting with the first step $d$.
     \IF{$n>1$ }
        \IF{$d=horizontal$ }
             \IF{$direction = NW$}
                \STATE  2 left.
            \ELSE
                \STATE 2 right.
            \ENDIF
            \STATE Init-Stair($n-1$, $vertical$, $direction$).
        \ELSE
             \IF {$direction = NW$}
                \STATE  2 up.
             \ELSE
                \STATE 2 down.
            \ENDIF
            \STATE Init-Stair($n-1$, $horizontal$, $direction$).
        \ENDIF
    \ENDIF
\end{algorithmic}
\label{alg:Init-Stair}
\end{algorithm}

\begin{algorithm}[H]
\caption{Middle-robots($i,\;n$)}
\begin{algorithmic}
  \STATE {\bf Input:} Let $k=4r$ the total number of robots, $i$ - the number of the current robot and let $n$ be the covered distance.
  \STATE {\bf Output:} The parallel search strategy of $r-2$ (middle) robots in a quadrant.
    \STATE Initialization($r-i+1, 5*(i-1)$). % coordinates from where the robot walks 2 up, 1 right, 2 down, 3 right.
    \FOR{$v=1$ \TO $n$ }  %how far the robot goes
         \FOR{$j=1$ \TO $2(r-i)$ }     %stable horizontal
            \STATE Stairs($3+8(v-1)$, $horizontal$, $NW$). %3, 11, 19, ... l
            \STATE Stairs($5+8(v-1)$, $horizontal$, $SE$). %5, 13, 21, ... l+2
         \ENDFOR
        \STATE Stairs($4+8(v-1)$, $horizontal$, $NW$). %4, 12...l+1 %grow
        \STATE Stairs($6+8(v-1)$, $vertical$, $SE$). %6, 14...l+3
        \STATE Stairs($8+8(v-1)$, $horizontal$, $NW$). %8, 16...l+5
        %stable vertical
         \FOR{$j=1$ \TO 2(i-1) }
            \STATE Stairs($7+8(v-1)$, $vertical$, $SE$). %7, 15...l+4
            \STATE Stairs($9+8(v-1)$, $vertical$, $NW$). %9, 17...l+6
        \ENDFOR
        %grow
        \STATE Stairs($10+8(v-1)$, $vertical$, $SE$). %10, 18 ...l+7
    \ENDFOR
\end{algorithmic}
\label{alg:Middle}
\end{algorithm}

\begin{algorithm}
\caption{Robot-r($n$)}
\begin{algorithmic}
  \STATE {\bf Input:} Let $k=4r$ the total number of robots and let $n$ be the covered distance.
  \STATE {\bf Output:} The parallel search strategy of the $r$th robot in a quadrant.
    \STATE Initialization(1, $5(r-1)$);\\% coordinates from where the robot walks 2 up, 1 right, 2 down, 3 right.
    %growing
        \STATE Stairs($8(v-1)+4$, $horizontal$, $NW$).
        \STATE Stairs($8(v-1)+6$, $vertical$, $SE$).
        \STATE Stairs($8(v-1)+8$, $horizontal$, $NW$).
        %stable vertical
         \FOR{$j=1$ \TO $2(r-1)$}
            \STATE Stairs($8(v-1)+7$, $vertical$, $SE$). %7 15? ...
            \STATE Stairs($8(v-1)+9$, $vertical$, $NW$). %9 17? ...
        \ENDFOR
    \FOR{$v=2$ \TO $n$} %how far the robot goes
        %growing
         \FOR{$k=1$ \TO $2$}
            \STATE Stairs($8(v-1)+4k-2$, $vertical$, $SE$).
            \STATE Stairs($8(v-1)+4k$, $horizontal$, $NW$).
        \ENDFOR
        %stable vertical
         \FOR{$j=1$ \TO $2(r-1)$}
            \STATE Stairs($8(v-1)+7$, $vertical$, $SE$). %7 15? ...
            \STATE Stairs($8(v-1)+9$, $vertical$, $NW$). %9 17? ...
        \ENDFOR
    \ENDFOR
\end{algorithmic}
\label{alg:Robot-p}
\end{algorithm}


\begin{thebibliography}{99}
\newlength{\bibskip}
\setlength{\bibskip}{1pt}

%{\fontsize{11}{13}\selectfont
\bibitem{Alpern}  S.~Alpern and S.~Gal. {\em The Theory of Search Games
               and Rendezvous}, Kluwer, 2002.

\bibitem{Baeza}  R.~Baeza-Yates, J.~Culberson and G.~Rawlins.
              {\em Searching in the plane}, Inf. and Comp., vol. {\bf 106}, (1993),
              pp.~234-252.

\vspace{\bibskip}


\bibitem{Baeza2} R.~Baeza-Yates, R.~Schott.
              {\em Parallel searching in the plane}, Comp. Geom.: Theory and Applications,
              Vol. {\bf 5}, (1995), pp.~143-154.

%\vspace{\bibskip}

%\bibitem{Bar-Elie} E.~Bar-Eli, P.~Berman, A.~Fiat and Peiyuan Yan.
%              ``On-line navigation in a room'',
%              {\em Proc. 3th ACM-SIAM Symp. on Discrete
%              Algorithms (SODA)}, (1992), pp.~237-249.

%\vspace{\bibskip}



%\bibitem{Beck} A.~Beck.
%              ``On the linear search problem'',
%              {\em Israel J. of Mathematics}, vol. {\bf 2}, (1964),
%              pp.~221-228.

%\vspace{\bibskip}
%\bibitem{Beck:more} A.~Beck.
%              ``More on the linear search Problem'',
%              {\em Israel J. of Mathematics}, vol. {\bf 3}, (1965),
%              pp.~61-70.


%\vspace{\bibskip}



%\bibitem{Beck:yet-more} A.~Beck and D.~J.~Newman.
%              ``Yet more on the linear search problem'',
%              {\em Israel J. of Mathematics}, Vol. {\bf 8}, (1970),
%              pp.~419-429.

%\vspace{\bibskip}



%\bibitem{BC}  A.~Blum and P.~Chalasani.
%              ``An on-line algorithm for improving performance in navigation''
%              {\em Proc. 34th IEEE Symp. on
%              Found. of Comp. Sci. (FOCS)}, (1993), pp.~2-11.

%\vspace{\bibskip}

%\bibitem{BRS} A.~Blum, P.~Raghavan and B.~Schieber.
%              ``Navigating in unfamiliar geometric terrain'',
%              {\em Proc.  23rd ACM Symp. on Theory of Computing (STOC)},
%              (1991), pp.~494-504.

\vspace{\bibskip}

\bibitem{ClearPath} Clear Path Robotics Inc.,
              http://www.clearpathrobotics.com/

\vspace{\bibskip}

\bibitem{CoastGuard} Canadian Coast Guard/Garde Cotiere Canadienne.
              {\em Merchant ship search and rescue manual}, (CANMERSAR),
              (1986).

%\vspace{\bibskip}


%\bibitem{Dobbie} J.~M.~Dobbie.
%              ``A survey of search theory'',
%              {\em Operations Research},
%              vol. {\bf 16}, (1968), pp.~525-537.

\vspace{\bibskip}

\bibitem{IMO1} IMO
              {\em IAMSAR Manual. Organization and Management},
              vol. {\bf I}, (2010).

\vspace{\bibskip}

\bibitem{IMO2} IMO
              {\em IAMSAR Manual. Mission Co-Ordination},
              vol. {\bf II}, (2010).

\vspace{\bibskip}

\bibitem{IMO3} IMO
              {\em IAMSAR Manual. Mobile Facilities},
              vol. {\bf III}, (2010).


%\vspace{\bibskip}
%
%\bibitem{Gal:book} S.~Gal.
%               {\em Search Games},
%               Academic Press, 1980.


%\vspace{\bibskip}
%\bibitem{Gal:general-search} S.~Gal.
%              ``A general search game'',
%              {\em Israel J. of Mathematics}
%              Vol. {\bf 12}, (1972), pp.~32-45.

%\vspace{\bibskip}


%\bibitem{Ick:thesis} Ch.~Icking. ``Motion and visibility in simple
%               polygons''.
%               Ph.D.\ Thesis, Fernuniversit\"at
%               Hagen, 1994.

%\vspace{\bibskip}

%\bibitem{Kao}  M.-Y. Kao, J.~H.~Reif and S.~R.~Tate.
%              ``Searching in an unknown environment: An optimal randomized
%              algorithm for the cow-path problem'',
%              {\em Proc.  4th ACM-SIAM Symp. on Discrete
%              Algorithms (SODA)}, (1993), pp.~441-447.

\vspace{\bibskip}

\bibitem{Koopman}  B.~O.~Koopman,
              {\em Search and screening}, Report No. 56 (ATI 64 627),
              Operations Evaluation Group,
              Office of the Chief of Naval Operations,
              Washington, D.C., 1946.


\vspace{\bibskip}


\bibitem{Lop:thesis}
              A.~L\'opez-Ortiz, {\em On-line target searching in bounded and
              unbounded domains}, Ph.D.\ thesis, University of
              Waterloo, 1996.

\vspace{\bibskip}

%\bibitem{animation} A.~L\'opez-Ortiz and D.~Maftuleac.
%                {\em Animations of search strategies}, https://cs.uwaterloo.ca/$\sim$alopez-o/search$\_$patterns.html


\vspace{\bibskip}

\bibitem{LopezOrtizSweet} A.~L\'opez-Ortiz and G.~Sweet,
    {\em Parallel Searching on a Lattice}, Proc. of the 13th Canadian Conference
     on Computational Geometry (CCCG), (2001).

%\vspace{\bibskip}


%\bibitem{los-ussmr-98}
%A.~L\'{o}pez-Ortiz and S.~Schuierer.
%\newblock ``The ultimate strategy to search on $m$ rays?''.
%\newblock {\em Proc.
% 4th Int.
%  Conf. on Computing and Combinatorics}, {\em LNCS} vol. 1449  ,
%  pp. 75--84, 1998.

\vspace{\bibskip}

\bibitem{sarscene} National Search and Rescue Secretariat/Secr\'etariat national
                Recherche et sauvetage.
               {\em CANSARP},
               SAR{Scene}, Vol. {\bf 4}, July 1994.
               
\vspace{\bibskip}

\bibitem{orlin} J.B.~Orlin, {\em A polynomial time primal network simplex algorithm for minimum cost flows}, Mathematical Programming 78 (1997) pp.~109–129.
               
               
\vspace{\bibskip}
\bibitem{comb-opt} C.H.~Papadimitriou and K.~Steiglitz, {\em Combinatorial Optimization: Algorithms and Complexity}, Dover Publications INC, 1998.


\vspace{\bibskip}

\bibitem{Stone} L.D.~Stone, {\em Revisiting the SS Central America Search},  International Conference on Information Fusion (FUSION) 2010, pp.~1-8.



\vspace{\bibskip}

\bibitem{USCoastGuardAdd} U.S. Coast Guard Addendum to the U.S. National Search and
             Rescue Supplement (NSS) to the IAMSAR Manual.
               COMDTINST M16130.2F, January 2013.


%}
\end{thebibliography}
\end{document}